\documentclass{article}

\usepackage{lipsum}
\usepackage{microtype}
\usepackage{graphicx}
\usepackage{color, colortbl}
\usepackage[dvipsnames]{xcolor}
\usepackage{subcaption}
\usepackage{booktabs} 
\usepackage{multirow}

\usepackage{hyperref}

\usepackage[preprint]{icml2026}

\usepackage{amsmath}
\usepackage{amssymb}
\usepackage{mathtools}
\usepackage{amsthm}

\usepackage[capitalize,noabbrev]{cleveref}

\theoremstyle{plain}
\newtheorem{theorem}{Theorem}[section]
\newtheorem{proposition}[theorem]{Proposition}
\newtheorem{lemma}[theorem]{Lemma}

\theoremstyle{definition}
\newtheorem{definition}[theorem]{Definition}

\theoremstyle{remark}

\usepackage[textsize=tiny]{todonotes}

\definecolor{Gray}{gray}{0.93}

\icmltitlerunning{Submission and Formatting Instructions for ICML 2026}

\begin{document}

\twocolumn[
    \icmltitle{
        \texorpdfstring{
            Scale-Consistent State-Space Dynamics \\ via Fractal of Stationary Transformations
        }{
            Scale-Consistent State-Space Dynamics via Fractal of Stationary Transformations
        }
    }

    
    
    \icmlsetsymbol{equal}{*}
    
    \begin{icmlauthorlist}
    \icmlauthor{Geunhyeok Yu}{khu}
    \icmlauthor{Hyoseok Hwang}{khu}
    \end{icmlauthorlist}
    
    \icmlaffiliation{khu}{Department of Software Convergence, Kyung Hee University}
    
    \icmlcorrespondingauthor{Hyoseok Hwang}{hyoseok@khu.ac.kr}
    
    \icmlkeywords{Machine Learning, ICML}
    
    \vskip 0.3in
]



\printAffiliationsAndNotice{}  

\begin{abstract}
Recent deep learning models increasingly rely on depth without structural guarantees on the validity of intermediate representations, rendering early stopping and adaptive computation ill-posed.
We address this limitation by formulating a structural requirement for state-space model's scale-consistent latent dynamics across iterative refinement, and derive Fractal of Stationary Transformations (FROST), which enforces a self-similar representation manifold through a fractal inductive bias.
Under this geometry, intermediate states correspond to different resolutions of a shared representation, and we provide a geometric analysis establishing contraction and stable convergence across iterations.
As a consequence of this scale-consistent structure, halting naturally admits a ranking-based formulation driven by intrinsic feature quality rather than extrinsic objectives.
Controlled experiments on ImageNet-100 empirically verify the predicted scale-consistent behavior, showing that adaptive efficiency emerges from the aligned latent geometry.
\end{abstract}


\section{Introduction}
\label{sec:introduction}

Modern deep learning models increasingly rely on depth to improve performance~\cite{scaling-kaplan, scaling-law-01, convnext, masked-autoencoders}, yet depth alone provides no structural guarantee that intermediate representations are geometrically valid approximations of the final output.
As networks grow deeper, latent states are produced through a long sequence of transformations without explicit constraints enforcing consistency across iterations.
Consequently, intermediate representations may drift away from the distribution of the final representation~\cite{depth-similarity, depth-spectral}, undermining their semantic validity and rendering early stopping ill-defined.
Under such conditions, depth acts as an unstructured accumulation of transformations rather than a principled mechanism for iterative refinement.

Recurrent models appear, at least superficially, to offer an escape from this dilemma~\cite{recurrent-visual-attention, renet, perceiver}.
By reusing parameters across iterations, they promise unbounded depth with finite memory.
However, this promise has largely failed to materialize in practice.
Existing recurrent architectures exhibit two structural limitations: their representational capacity often falls short of deep feed-forward transformers, and more critically, they lack a principled mechanism for adaptive depth.
Despite weight sharing, most recurrent models execute a fixed number of iterations, treating all inputs as equally complex.
As a result, recurrence reduces parameter redundancy but does not address the absence of a geometric criterion for meaningful early stopping.

Adaptive computation methods attempt to address this limitation by introducing early-exit mechanisms.
Some approaches predict halting decisions via auxiliary objectives or specialized modules~\cite{branchynet, skipnet, msdnet, universal-transformers, pondernet, mixture-of-depths}.
While effective in isolated cases, these methods rely on extrinsic signals that are potentially orthogonal to the model's primary learning objective.
The halting decision is learned as a separate task, often competing with representation learning and destabilizing optimization.
More importantly, these methods treat intermediate representations as incomplete artifacts, rather than principled approximations of the final output.
This mismatch reveals a deeper issue: existing models lack a structural reason why stopping early should be correct.
See Appendix~\ref{appndx:related-works} for more comparison.

We argue that adaptive computation cannot be reliably achieved without rethinking the geometry of latent representations.
Early halting is only meaningful if intermediate states lie on the same distribution as the final representation.
Without such a guarantee, halting becomes a heuristic layered atop an incompatible model.

In this work, we formalize scale-consistent latent dynamics in state-space models~(\S\ref{sec:frost}) through what we term Fractal of Stationary Transformations~(FROST).
This formulation imposes a fractal inductive bias~\cite{mandelbrot-fractal, fractal-geometry, fractal-analysis}, under which latent representations become self-similar across iterative refinement.
Under this structure, intermediate states correspond to different resolutions of a shared underlying representation, rather than incomplete surrogates produced along an optimization trajectory.
This property enables a theoretically infinite representational manifold under finite space complexity, while ensuring that all depths remain geometrically aligned.

Crucially, the self-similar structure induced by FROST renders halting an intrinsic property of the model.
Under the resulting scale-consistent geometry, halting naturally admits a ranking-based formulation that leverages the fractal trajectory itself~(\S\ref{sec:halting}).
Halting decisions follow from the relative quality of representations, aligning computational efficiency with the primary task objective without auxiliary supervision.
As a result, FROST avoids the instability and overhead characteristic of prior adaptive computation methods.

Our contributions are summarized as follows:

\begin{itemize}
    \item We formulate a fractal inductive bias for state-space models by enforcing scale-consistent latent dynamics.
    \item We analyze the resulting geometry, establishing contraction and stable convergence onto a self-similar latent manifold.
    \item We show that halting naturally admits a ranking-based formulation under the proposed geometric structure.
    \item We empirically validate the theoretical predictions, confirming the anticipated convergence behavior in finite neural implementations.
\end{itemize}


\section{Preliminaries}
\label{sec:preliminaries}

This section introduces the notation and background used throughout the paper. 
We focus on standard definitions and formulations to establish the necessary background for subsequent sections.

\subsection{Notation and Conventions}
Let $t \in \mathbb{N}$ denote a discrete iteration index and $T \in \mathbb{N}$ the maximum number of iterations during inference. 
Inputs are denoted by $x_t \in \mathbb{R}^{D_{\mathrm{in}}}$, latent states by $h_t \in \mathbb{R}^{D_{\mathrm{hid}}}$, and outputs by $y_t \in \mathbb{R}^{D_{\mathrm{out}}}$. Lowercase letters indicate time-indexed variables, while uppercase letters denote linear or nonlinear maps.

For continuous-time formulations, we write $h(t)$ and $x(t)$ with derivatives $\dot{h}(t)=\frac{d}{dt}h(t)$. 
All vector norms are Euclidean unless stated otherwise. We assume sufficient regularity of all mappings to ensure well-defined solutions.

\subsection{Continuous-Time State-Space Models}
A continuous-time SSM~\cite{ssm-s4, ssm-s3, ssm-hyena} is defined by
\begin{equation}
\dot{h}(t) = A(h(t)) + B(x(t)),
\end{equation}
with a corresponding output
\begin{equation}
y(t) = C(h(t)) + D(x(t)).
\end{equation}
Here, $A$ denotes the latent transition map, $B$ the input map, and $C,D$ the readout maps. 
All mappings are time-invariant. 
We introduce continuous-time dynamics as a conceptual tool to analyze structural properties, not as a claim of physical time evolution.

\subsection{Discretization via Induced Dynamics}
In practice, SSMs operate in discrete time. 
Rather than approximating the differential equation directly, we consider an induced continuous-time system whose vector field is piecewise constant on unit-length intervals. 
For integer $t$,
\begin{equation}
\dot{h}(\tau) := A(h(t)) + B(x(t)), \quad \tau \in [t, t+1).
\end{equation}
Under this construction, the dynamics are exactly integrable on each interval.

\begin{lemma}[Exact discrete update of the induced dynamics]
\label{lem:exact_discretization}
Let $h(\tau)$ satisfy the induced dynamics on $[t,t+1)$ with initial condition $h(t)$. Then
\begin{equation}
h(t+1) = h(t) + A(h(t)) + B(x(t)).
\end{equation}
\end{lemma}
\begin{proof}
Since the vector field is constant on $[t,t+1)$, the integral over the interval evaluates exactly to the stated update.
\end{proof}

Lemma~\ref{lem:exact_discretization} shows that the residual-style discrete update corresponds to the exact one-step flow of a well-defined induced system, enabling continuous-time analysis of discrete residual networks.

\subsection{Fractals and Self-Similarity}
A fractal object exhibits self-similarity, meaning that its structure is preserved across scales. 
In stochastic settings, self-similarity is characterized by the Hurst exponent $H$. 
A process $X(t)$ is $H$-self-similar if, for any $a>0$,
\begin{equation}
X(at) \overset{d}{=} a^{H} X(t),
\end{equation}
where $\overset{d}{=}$ denotes equality in distribution.

Although network updates are deterministic, inputs are drawn from a data distribution, inducing a stochastic latent trajectory $h(t)$. 
We adopt this statistical notion of self-similarity to describe how the distribution of latent representations scales with depth or iteration.

\subsection{Halting and Adaptive Computation}
Adaptive computation allows iterative models to terminate inference at different steps depending on the input. 
Consider an iterative model producing intermediate representations $h_t$ and corresponding outputs $y_t$ at each step $t$. 
In an anytime prediction setting, inference may be stopped at any intermediate step, trading accuracy for computational cost.

Crucially, the effect of halting depends not only on when computation stops, but also on how representations evolve across iterations. 
If intermediate representations lie on unrelated or inconsistent manifolds, early stopping merely truncates computation and inevitably degrades representational fidelity. 
In such cases, halting functions only as a computational shortcut.

In contrast, when latent representations evolve along a structured and self-similar manifold, halting corresponds to selecting a coarser or finer resolution of the same underlying representation. 
From this perspective, adaptive computation can be interpreted as a mechanism for controlling representational fidelity, rather than as an auxiliary heuristic for saving computation. 
This viewpoint motivates the halting formulation introduced in \S\ref{sec:halting}.


\section{Fractal of Stationary Transformations}
\label{sec:frost}

In this section, we introduce FROST, a state-space formulation that decouples representational complexity from model size by imposing a fractal inductive bias on iterative dynamics. 
Building on the preliminaries in \S\ref{sec:preliminaries}, we first formalize the notion of scale-consistent refinement, then derive the discrete-time realization using general functional forms to accommodate potential nonlinearities, and finally demonstrate how stationary transformations induce a fractal structure.
For theoretical proofs, please refer to Appendix~\ref{appndx:theory}.

\subsection{Scale-Consistent Refinement}
\begin{definition}[Scale-consistent refinement]
An iterative SSM is said to admit scale-consistent refinement if successive iterations correspond to observing a shared representation structure at progressively finer resolutions, rather than expanding or altering the underlying representation space.
\label{def:scale-consistency}
\end{definition}
This perspective aligns with continuous-depth frameworks~\cite{neural-ode, deep-equilibrium-models}, yet explicitly enforces a fractal structure across discrete steps.
Iterative SSMs construct representations through repeated application of a fixed transformation. 
While expressive, standard formulations implicitly fix the structure of the representation space via the model parameterization. 
Increasing the number of iterations refines representations within this fixed structure but does not introduce a principled notion of scale.

Our key observation is that iterative refinement can be reinterpreted as a form of progressive zooming into a shared representation space. 
To make this interpretation explicit, the underlying transformation must be consistent across scales: refining the representation should correspond to observing the same structure at a finer resolution. 
This motivates the introduction of a scale parameter that governs resolution without altering the underlying transformation.

\subsection{Stationary Transformations and Discretization}
\begin{definition}[Stationarity and contractivity]
A state update rule is called stationary if the same transformation is applied at every iteration, and contractive if there exists a scaling factor $\lambda \in (0,1]$ controlling the magnitude of refinement without changing the operating point of the dynamics.\footnote{We use the term `contractive' to refer to the interpolation weight $\lambda$, distinct from the Banach fixed-point definition.}
\label{def:stationary_contraction}
\end{definition}

We begin by specializing the general continuous-time formulation to stationary SSMs. 
Consider the time-invariant system:
\begin{equation}
\dot{h}(t) = \mathcal{F}(h(t), x(t)) := A(h(t)) + B(x(t)),
\label{eq:basic-ssm}
\end{equation}
where $A$ is the state transition operator and $B$ is the input mapping operator. 
We use the notation $A(\cdot)$ and $B(\cdot)$ to emphasize that these may be general nonlinear functions, generalizing the matrix-vector multiplication of linear SSMs.

Based on the induced dynamics derived in Lemma~\ref{lem:exact_discretization}, we formulate the discrete update rule. 
Rather than treating this as a simple residual addition, we interpret the update as a first-order descent-like step on the latent manifold, where $\lambda$ acts as a learnable step size. 
This interpretation is conceptual and does not assume the existence of a globally defined potential function.

Specifically, the update direction $A(h_t) + B(x_t)$ can be interpreted as a local descent direction, yielding
\begin{equation}
h_{t+1} \approx h_t - \lambda \nabla \mathcal{E}(h_t),
\end{equation}
for some implicit local energy $\mathcal{E}$ defined in a neighborhood of $h_t$. 
In this view, $\lambda$ controls the refinement rate toward a locally stable configuration, providing intuition for iterative convergence without requiring global energy minimization.
Here, $\lambda$ governs the convergence rate towards a low-energy equilibrium state, framing the iterative refinement as a dynamic energy minimization process.

To ensure numerical stability and maintain a consistent operating point across scales, we rewrite this as a spectral gating operation. 
Throughout this section, the scaling factor $\lambda$ serves as a unified resolution parameter that simultaneously controls refinement magnitude, spectral stability, and the effective scale at which the latent manifold is observed:
\begin{equation}
\label{eq:discrete_interpolation}
h_{t+1} = (1 - \lambda) h_t + \lambda \bigl( h_t + \mathcal{F}(h_t, x_t) \bigr).
\end{equation}
This formulation allows us to analyze the update through the lens of spectral control. 
By interpolating between the identity operator (preserving the current state) and the update operator, $\lambda$ effectively modulates the eigenvalues of the Jacobian $\partial h_{t+1} / \partial h_t$. 
A smaller $\lambda$ pulls the spectrum towards unity, improving condition number and gradient flow (stability), while a larger $\lambda$ allows for rapid state transitions (plasticity). 
This mechanism provides a principled way to balance memory retention and new information integration within a fixed-width architecture.

\subsection{Fractal Inductive Bias via Self-Similarity}
\begin{lemma}[Scaling equivariance of induced dynamics]
Under stationary contractive transformations, the discrete state-space dynamics are equivariant under joint scaling of time and input, preserving their functional form up to scale-dependent factors.
\end{lemma}

We now formalize how stationary scaling induces self-similarity. 
Following the continuous-time framework, we apply $\lambda > 0$ and define scaled input, output, and latent processes as:
\begin{align}
x_{\lambda}(t) &:= \lambda^{-H} x(\lambda t), \quad
y_{\lambda}(t)  := \lambda^{-H} y(\lambda t), \\
h_{\lambda}(t) &:= h(\lambda t), 
\end{align}
where $H \in (0,1]$ denotes the Hurst exponent.

\paragraph{Scaled State Transition.}
Substituting the scaled variables into the continuous dynamics $\dot{h}(t) = \mathcal{F}(h(t), x(t))$ yields:
\begin{equation}
\frac{d}{dt} h_{\lambda}(t) = \lambda \dot{h}(\lambda t) = \lambda \mathcal{F}(h(\lambda t), x(\lambda t)).
\end{equation}
Expressing this in terms of the scaled variables $x_\lambda$ and $h_\lambda$:
\begin{align}
\frac{d}{dt} h_{\lambda}(t) &= \lambda A(h_{\lambda}(t)) + \lambda B(\lambda^H x_{\lambda}(t)).
\end{align}
This reveals the self-similar structure of the operators. 
We define the scaled transition and input operators as:
\begin{equation}
A_{\lambda}(h) := \lambda A(h), \qquad B_{\lambda}(x) := \lambda B(\lambda^{H} x).
\label{eq:scaled_operators}
\end{equation}
This definition ensures that the scaled dynamics maintain the form $\dot{h}_\lambda = A_\lambda(h_\lambda) + B_\lambda(x_\lambda)$. 
For linear $B$, this simplifies to $B_\lambda(x) = \lambda^{1+H} B(x)$.

\paragraph{Discrete Realization.}
Substituting these scaled operators into the discrete update rule yields the fractal update equation:
\begin{equation}
h_{t+1} = (1-\lambda)h_t + \Delta h_\lambda(h_t, x_t),
\label{eq:discrete_realization}
\end{equation}
where the update term is rigorously defined using the scaled operators from Eq.~\eqref{eq:scaled_operators}:
\begin{equation}
\Delta h_\lambda(h_t, x_t) = \lambda h_t + A_\lambda(h_t) + B_\lambda(x_t).
\label{eq:frost-update-rule}
\end{equation}
Here, the term $\lambda h_t$ represents the scaled identity connection absorbed into the update. 
This formulation confirms that the discrete-time implementation naturally preserves the continuous scaling laws without requiring additional approximations.
By parameterizing the update directly via $\lambda$ and $H$, FROST enforces a fractal inductive bias that is structurally consistent across varying depths.

\paragraph{Output Scaling.}
For the output equation $y(t) = C(h(t)) + D(x(t))$, applying the scaling relations yields:
\begin{equation}
\lambda^H y_\lambda(t) = C(h_\lambda(t)) + D(\lambda^H x_\lambda(t)).
\end{equation}
Dividing by $\lambda^H$, we identify the scaled readout operators:
\begin{equation}
C_\lambda(h) := \lambda^{-H} C(h), \qquad D_\lambda(x) := \lambda^{-H} D(\lambda^H x).
\end{equation}
If $D$ is linear (or homogeneous of degree 1), then $D_\lambda(x) = D(x)$, meaning the feedthrough map remains invariant under this fractal scaling.

\begin{proposition}[Fractal inductive bias of stationary transformations]
Stationary contractive state-space transformations parametrized by $(A_\lambda, B_\lambda, C_\lambda, D_\lambda)$ induce self-similar dynamics. 
Consequently, iterative refinement with these parameters corresponds to observing a shared transformation at different resolutions determined by $\lambda$.
\label{prop:fractal-inductive_bias}
\end{proposition}

Under this formulation, fractality is defined by scale-invariant refinement under stationary contraction, with $\lambda$ and $H$ specifying the scale ratio and self-similarity profile that preserve this invariance across iterations.

\subsection{Variable-Complexity Inference}

A direct consequence of this fractal formulation is that representational complexity becomes a function of resolution (depth) rather than model size. 
Increasing the number of iterations $T$ simply extends the trajectory generated by the self-similar rule, refining the representation without expanding the parameterization (since $A, B, C, D$ are shared).

This property naturally supports variable-complexity inference. 
Because all intermediate states lie on the same self-similar manifold, halting does not truncate computation arbitrarily but selects a coarser or finer approximation of a shared representation. 
In this sense, halting acts as a control over representational fidelity rather than a heuristic for saving computation. 
This principle naturally gives rise to a ranking-based halting strategy~(\S\ref{sec:halting}).


\section{Ranking-based Halting}
\label{sec:halting}

In this section, we describe a ranking-based halting strategy that operationalizes the variable-complexity property of FROST. 
Rather than determining stopping points via absolute thresholds or regression targets, our approach allocates refinement depth based on the relative difficulty of intermediate predictions within a batch.

\subsection{Per-iteration Prediction and Difficulty}
At each iteration $t$, the model produces an intermediate prediction $y_{i,t}$ for each sample $x_i$ in a mini-batch of size $B$, based on the current latent state $h_{i,t}$. 
Let $\ell_{i,t}$ denote a task-specific loss evaluated at iteration $t$ for sample $x_i$. 
Across $T$ iterations, this yields a collection of $BT$ losses per mini-batch.

\subsection{Batch--Time Ranking}
Instead of aggregating losses across iterations to obtain a single difficulty score per sample, we rank the entire set
\begin{equation}
\mathbf{L} := \{ \ell_{i,t} \mid i=1,\ldots,B,\ t=1,\ldots,T \}.
\end{equation}
Let $r_{i,t} \in \{1,\ldots,BT\}$ denote the rank of the loss $\ell_{i,t}$ within the sorted set $\mathcal{S}_{\text{loss}}$ (ascending order). 
This ranking provides a scale-free notion of relative difficulty that is invariant to monotone transformations of the loss, ensuring robustness against loss scale variations during training.

\subsection{Ranking-based Supervision}
We use the rank information to shape a halting signal without requiring absolute thresholds or ground-truth stopping indices. 
At each iteration $t$, a halting score $s_{i,t} \in [0,1]$ is predicted from the current latent state $h_{i,t}$ via a linear map:
\begin{equation}
s_{i,t} = \sigma(\mathbf{w}_s^\top h_{i,t} + b_s),
\end{equation}
where $\sigma$ is the sigmoid function and $\mathbf{w}_s, b_s$ are learnable parameters.

Training encourages $s_{i,t}$ to be larger for low-rank (easier) pairs $(i,t)$ and smaller for high-rank (harder) pairs, thereby allocating fewer refinement steps to representations that become reliable earlier. 
This strategy disentangles the task training objective from the halting objective, isolating the backbone parameters $A(\cdot)$ from potential disruption caused by auxiliary halting gradients.

\subsection{Resolution Selection at Inference}
At inference time, resolution selection is performed by comparing halting scores against an empirical quantile estimated from previously observed scores. 
Threshold calibration is notoriously difficult in adaptive computation due to score drift during training. 
To address this, we accumulate the predicted halting scores $s_{i,t}$ across batches and iterations during training and validation, forming a streaming distribution that reflects refinement confidence over time.

To estimate quantiles of this distribution efficiently, we maintain a quantile sketch using the KLL~\cite{kll} algorithm. 
Given a target percentile $q \in (0,1)$, the sketch provides an estimate $s_\text{halt}$ such that
\begin{equation}
\mathbb{P}(s \leq s_\text{halt}) \approx q
\end{equation}
under the empirical distribution of halting scores.

At inference time, computation at iteration $t$ for input $x_i$ is halted if the score $s_{i,t}$ exceeds the estimated threshold $s_\text{halt}$. 
Samples whose scores remain below this threshold continue refinement for additional iterations. 
This percentile-based criterion avoids reliance on absolute score thresholds and ensures robustness to scale and distributional shifts in the halting signal. 
Importantly, the use of a streaming quantile estimator decouples the inference rule from batch composition and allows resolution selection to remain consistent with the ranking-based supervision used during training.


\section{Training Objective}
\label{sec:training}

We describe the training objective used to learn FROST with ranking-based halting. The overall objective augments a standard task loss with an auxiliary ranking loss that supervises the halting signal, while leaving the task formulation unchanged.

\subsection{Task Loss}
Let $y_{i,t}$ denote the intermediate prediction produced at iteration $t$ for input $x_i$, and let $\ell_{i,t} = \mathcal{L}_{\text{task}}(y_{i,t}, y_i^{\ast})$ be the task-specific loss (e.g., cross-entropy for classification), where $y_i^{\ast}$ denotes the ground-truth. 
During training, the model is unrolled for all $T$ iterations; however, the task loss is backpropagated only from the final iteration. 
Concretely, we define
\begin{equation}
\mathcal{L}_{\text{task}} := \frac{1}{B} \sum_{i=1}^{B} \ell_{i,T}.
\end{equation}
This design reflects the theoretical property of FROST: since all iterations apply the same stationary transformation, the representation geometry is shared across depths. 
As a result, supervising only the final iteration is sufficient to learn task-relevant structure, while intermediate iterations are shaped indirectly through the shared dynamics. 
This objective optimizes the scale factor $\lambda$ to control the contraction rate while fitting the backbone parameters to the task.

\subsection{Ranking Loss for Halting}
\label{ssec:rank-loss}
To supervise the halting signal, we introduce a ranking loss defined over the batch--time loss set $\mathbf{L}$. Let $r_{i,t}$ denote the rank of $\ell_{i,t}$ among the $BT$ losses in the batch, as defined in \S\ref{sec:halting}. 
The ranking loss is composed of two complementary terms: relative ranking and absolute anchoring losses.

\paragraph{Relative ranking loss.}
The relative loss $\mathcal{L}_{\text{rank}}^{\text{rel}}$ enforces correct ordering between easy and hard samples. 
Following the ranking-based halting formulation, we construct sets of easy ($E$) and hard ($H$) samples based on the rank ordering (e.g., top-k vs. bottom-k) and apply a margin-based ranking loss~\cite{ranking-loss} across all iterations:
\begin{equation}
\mathcal{L}_{\text{rank}}^{\text{rel}} = \sum_{t=1}^{T} \frac{1}{|E||H|} \sum_{(i,j)\in E\times H} \max(0, s_{j,t} - s_{i,t} + \delta),
\end{equation}
where $\delta$ is a margin hyperparameter. 
This term captures relative difficulty and is invariant to absolute score scales.

\paragraph{Absolute anchoring loss.}
While the relative loss enforces ordering, it does not constrain the absolute distribution of halting scores and may lead to score drift. 
To stabilize training, we introduce an absolute anchoring loss $\mathcal{L}_{\text{rank}}^{\text{abs}}$ that polarizes the scores of easy and hard samples:
\begin{equation}
\mathcal{L}_{\text{rank}}^{\text{abs}} = \sum_{t=1}^{T} \frac{1}{|E|} \sum_{i\in E} \log{s_{i,t}} + \frac{1}{|H|} \sum_{j\in H} \log{1-s_{j,t}},
\end{equation}
This binary cross-entropy term anchors the halting scores to the extremes (0 and 1) while remaining subordinate to the relative objective.

\subsection{Overall Objective}
The final training objective combines the task loss with the two ranking losses as
\begin{equation}
\mathcal{L} = \mathcal{L}_{\text{task}} + \alpha_{\text{rel}} \mathcal{L}_{\text{rank}}^{\text{rel}} + \alpha_{\text{abs}} \mathcal{L}_{\text{rank}}^{\text{abs}},
\end{equation}
where $\alpha_{\text{rel}}$ and $\alpha_{\text{abs}}$ control the influence of relative ordering and absolute anchoring, respectively\footnote{We empirically tuned $\alpha_{\text{rel}}=0.7$ and $\alpha_{\text{abs}}=0.3$.}. 
In practice, $\mathcal{L}_{\text{rank}}^{\text{rel}}$ provides the primary supervision signal, while $\mathcal{L}_{\text{rank}}^{\text{abs}}$ acts as a regularizer to prevent score collapse.

Importantly, while the ranking losses propagate gradients to the backbone, they do not structurally conflict with the task objective. 
The ranking supervision purely encourages the halting scores to align with the relative difficulty already determined by the task loss. 
Consequently, this auxiliary supervision cooperates with the main objective, allowing the model to learn adaptive resource allocation without disrupting the representation geometry required for the task.


\section{Analysis}
\label{sec:analysis}

Given our focus on validating fundamental geometric dynamics rather than large-scale scaling behavior, we conduct our empirical analysis primarily on ImageNet-100~\cite{imagenet}, with complementary experiments on CIFAR-100~\cite{cifar}.
Our analysis focuses on confirming the fractal inductive bias, demonstrating the behavior of the ranking-based halting mechanism, and comparing the resulting dynamics against established baselines.
For the latent transition function ($A$), we consider ResNet~\cite{resnet} and Vision Transformer (ViT)~\cite{vit} architectures to cover a range of representative model classes.
Architectural details are provided in Appendix~\ref{appndx:architecture-details}.

\subsection{Scale-Consistency and Self-Similarity}
\label{ssec:analysis-fractal_consistency}

We analyze whether FROST learns a representation structure that is consistent across scales while simultaneously enriching fine-grained details, as predicted by the fractal inductive bias. 
Our analysis separates this phenomenon into two complementary aspects: (i) macroscopic scale-consistency, which examines whether intermediate representations preserve a shared structure across iterations, and (ii) microscopic self-similarity, which evaluates how much local complexity is encoded within this aligned structure.

\paragraph{Macroscopic scale-consistency.}
We first verify whether FROST learns a representation structure that is consistent across scales, as defined in Definition~\ref{def:scale-consistency}. 

\begin{table}[ht]
    \centering \scriptsize
    \newcommand{\std}[1]{\scriptsize{$\pm$#1}}
    \begin{tabular}{l|cccccc}
        \toprule
        Depth ($t$) & 1 & 3 & 6 & 9 & 12 & 15 \\
        \midrule
        Vanilla & 0.3843 & 0.4245 & 0.5104 & 0.6220 & 0.6558 & 0.8371 \\
                & \std{0.11} & \std{0.11} & \std{0.11} & \std{0.13} & \std{0.13} & \std{0.11} \\
        SSM     & 0.0539 & 0.3505 & 0.7144 & 0.9012 & 0.9767 & 0.9989 \\
                & \std{0.24} & \std{0.19} & \std{0.10} & \std{0.04} & \std{0.01} & \std{0.00} \\
        \textbf{FROST} & \textbf{0.2991} & \textbf{0.6969} & \textbf{0.9018} & \textbf{0.9684} & \textbf{0.9925} & \textbf{0.9996} \\
                & \std{0.19} & \std{0.12} & \std{0.04} & \std{0.01} & \std{0.00} & \std{0.00} \\
        \bottomrule
    \end{tabular}
    \caption{
        Mean and standard deviation of cosine similarity between the final representation $h_T$ and intermediate states $h_t$ over 10K samples. 
        Values closer to 1.0 indicate higher structural consistency with the final output.
        \vspace{-1em}
    }
    \label{tab:scale_consistency}
\end{table}

Table~\ref{tab:scale_consistency} reports the mean and standard deviation of cosine similarity over 10K samples. 
Vanilla baselines begin with moderately high similarity ($0.38$) but exhibit irregular progression, suggesting that early iterations often learn features directionally misaligned with the final representation. 
This inconsistency complicates the use of intermediate outputs for reliable anytime prediction. 
Standard SSMs, in contrast, show near-orthogonal initialization ($0.05$) with high variance ($\pm 0.24$), indicating unstable early dynamics and poor initial alignment.

FROST demonstrates the most desirable behavior. 
It starts with a moderate yet stable alignment ($0.30$) and increases similarity smoothly and consistently with depth while maintaining low variance. 
This monotonic progression indicates that successive iterations refine a shared representation geometry rather than altering its global structure, providing empirical evidence of scale-consistent refinement.

\paragraph{Microscopic self-similarity and complexity.}
While cosine similarity captures macroscopic alignment, it does not quantify how much structural detail is encoded within this aligned trajectory. 
To characterize the microscopic complexity of the latent dynamics, we estimate the Minkowski dimension $D$ of the latent trajectory~(Proposition~\ref{prop:fractal-inductive_bias}).

Both Vanilla and BasicSSM baselines exhibit dimensions of approximately $D \approx 1.262$, which is strikingly close to the theoretical dimension of a classical Koch curve ($D \approx 1.2619$). 
This suggests that, despite architectural differences, their dynamics are implicitly constrained to a relatively simple and well-known class of rough trajectories.

FROST, however, exhibits a significantly higher dimension ($D \approx 1.295$). 
The higher dimension implies that FROST learns a space-filling latent trajectory, enabling denser information packing within the same ambient dimensionality.

\subsection{Fractal Component Analysis}
\label{ssec:analysis-fractal_components}
We perform an ablation study on the key fractal parameters to understand their impact on model behavior.
\begin{figure}[ht]
    \centering
    \includegraphics[width=0.9\linewidth]{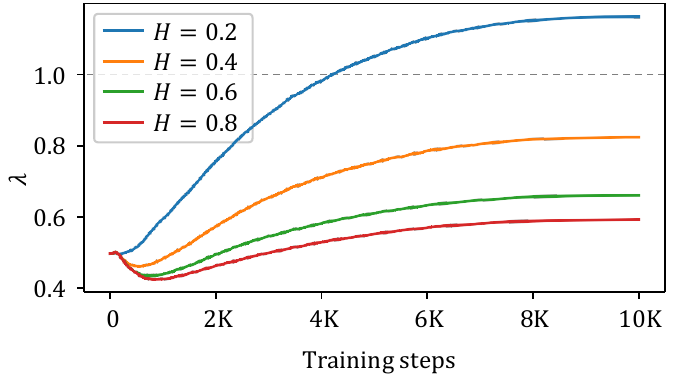}
    \vspace{-1em}
    \caption{
        Convergence behavior of $\lambda$ across different $H$.
        \vspace{-2em}
    }
    \label{fig:lambda}
\end{figure}

\paragraph{The $\lambda$ parameter.}
We investigate the learning dynamics of the contraction factor $\lambda$ (Definition~\ref{def:stationary_contraction}, $H=0.8$ in Figure~\ref{fig:lambda}). 
Initialized at $\lambda=0.5$, we observe a distinct two-phase convergence behavior.

In the first phase, $\lambda$ rapidly decreases to a minimum of $\lambda \approx 0.43$. 
This suggests the model initially prioritizes stability, taking smaller refinement steps to robustly learn the core representation manifold without large disruptions.
In the second phase, $\lambda$ gradually recovers and converges to an optimal value of $\lambda \approx 0.59$.
Crucially, throughout the entire training process, $\lambda$ remains strictly below 1.0, validating that the learned state transition naturally adheres to the contraction mapping required for stable fractal dynamics.

\paragraph{Hurst exponent ($H$).}
We observe an inverse correlation between $H$ and $\lambda$, where extreme roughness ($H=0.2$) drives $\lambda > 1.0$, prioritizing input reactivity over contraction. 
Conversely, treating $H$ as learnable causes it to diverge beyond $1.0$, exploiting the spectral bias to dampen high-frequency details ($B_\lambda \propto \lambda^{1+H}$). 
Thus, fixing $H$ is a critical structural constraint to prevent trivial oversmoothing and strictly enforce the fractal inductive bias.

\subsection{Halting Ability}
\label{ssec:analysis-halting}
We evaluate the ranking-based halting strategy described in \S\ref{sec:halting} using CIFAR-100 classification. 
Table~\ref{tab:haltability} shows the relationship between the target quantile $q$ (representing the proportion of samples allowed to exit) and the resulting average depth and accuracy. 

\begin{table}[ht]
    \centering \scriptsize
    \setlength{\tabcolsep}{1.15mm}
    \begin{tabular}{l|cccccccc}
        \toprule
        Quantile ($q$) & 12.50 & 25.00 & 37.50 & 50.00 & 62.50 & 75.00 & 87.50 & 100.00 \\ \midrule
        Depth ($t$) & 4.37 & 6.06 & 7.26 & 8.30 & 9.43 & 10.77 & 12.49 & 16.00 \\
        Accuracy (\%) & 45.39 & 52.40 & 55.78 & 57.72 & 59.47 & 60.70 & 61.77 & 61.99 \\
        \bottomrule
    \end{tabular}
    \caption{
        The depth and the CIFAR-100 classification accuracy across eight $q$-quantiles of $s$.
        \vspace{-1em}
    }
    \label{tab:haltability}
\end{table}

The results demonstrate a clear monotonic trend: as the quantile threshold increases (allowing harder samples to exit later), both the average depth and classification accuracy rise.
Notably, at lower quantiles ($q \le 25\%$), FROST achieves reasonable accuracy ($>52\%$) with significantly reduced depth ($t \approx 6.0$), effectively filtering out easy samples with minimal computation. 
Conversely, for the hardest subset ($q=100\%$), the model utilizes the full depth ($16.0$), achieving peak accuracy ($62\%$). 
This wide dynamic range in depth allocation confirms that our ranking-based mechanism successfully identifies sample difficulty and allocates computational resources adaptively, validating the efficiency claims of FROST.

\subsection{Loss Ablation}
\label{ssec:analysis-loss}
\begin{figure}[ht]
    \centering
    \newcommand{\figwidth}{0.325\linewidth}
    
    \begin{subfigure}{\figwidth}
    \includegraphics[width=\linewidth]{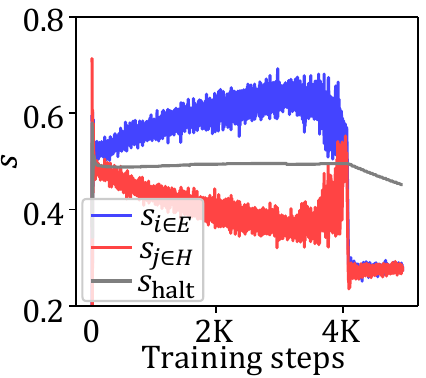}
    \caption{$\mathcal{L}_\text{rank}^\text{abs}$ only}
    \label{sfig:score-trajectory-abs}
    \end{subfigure}
    \begin{subfigure}{\figwidth}
    \includegraphics[width=\linewidth]{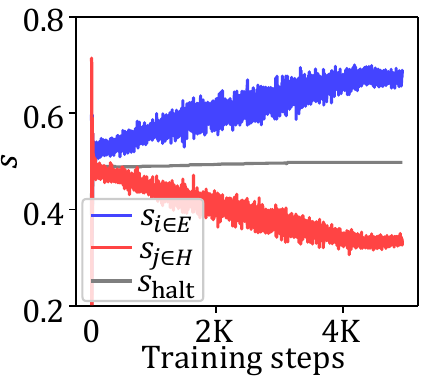}
    \caption{$\mathcal{L}_\text{rank}^\text{abs} + \mathcal{L}_\text{rank}^\text{rel}$}
    \label{sfig:score-trajectory-abs-rel}
    \end{subfigure}
    \begin{subfigure}{\figwidth}
    \includegraphics[width=\linewidth]{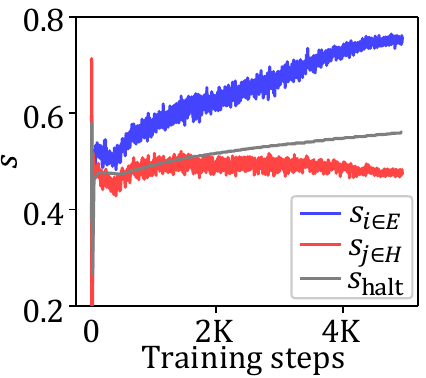}
    \caption{$\mathcal{L}_\text{rank}^\text{rel}$ only}
    \label{sfig:score-trajectory-rel}
    \end{subfigure}
    \caption{
        Convergence trajectory of $s$ varying $\mathcal{L}_\text{rank}$ composition.
        \vspace{-1.5em}
    }
    \label{fig:score-trajectory}
\end{figure}
We dissect the contribution of the ranking losses described in \S\ref{sec:training}. 
Figure~\ref{fig:score-trajectory} visualizes the convergence trajectory of halting scores during training for easy samples ($s_{i\in E}$), hard samples ($s_{j\in H}$), and the median halting threshold ($s_\text{halt}$).

As depicted in Figure~\ref{sfig:score-trajectory-abs}, relying exclusively on the absolute anchoring loss $\mathcal{L}_\text{rank}^\text{abs}$ forces the scores $s$ towards the extremes (0 or 1) rapidly in the initial steps. 
This suggests that the linear router mapping the hidden state to $s$ converges significantly faster than the complex backbone network, leading to a mutual collapse where the feature representation fails to mature (accuracy drops to $29.79\%$).

Conversely, Figure~\ref{sfig:score-trajectory-rel} demonstrates the limitation of using only the relative ranking loss $\mathcal{L}_\text{rank}^\text{rel}$. 
While it successfully separates the scores of easy ($s_{i \in E}$) and hard ($s_{j \in H}$) samples, it fails to utilize the full dynamic range $[0, 1]$, resulting in a compressed score distribution. 
This compression hinders fine-grained discrimination at inference time.

The combined loss, shown in Figure~\ref{sfig:score-trajectory-abs-rel}, strikes a stable balance: it maintains the correct relative order while effectively spreading the scores across the entire range. 
The quantitative impact is confirmed by CIFAR-100 classification. 
While the relative-only baseline achieves a competitive accuracy of $61.68\%$ at full depth ($q=1.0$), comparable to the combined model ($61.96\%$), the gap widens significantly in the early-exit regime. 
At a strict threshold of $q=0.25$, the combined objective outperforms the relative-only model by $11.20\%p$, demonstrating that utilizing the full score range is critical for reliable feature quality discrimination.

\subsection{Complexity Analysis}
\label{ssec:analysis-complexity}

We compare the computational complexity of ResNet and ViT-B/16 backbones in terms of parameter count, GFLOPs, and inference throughput~(Table~\ref{tab:complexity}).
\begin{table}[ht]
    \centering \scriptsize
    \setlength{\tabcolsep}{1.3mm}
    \begin{tabular}{llcccc}
        \toprule
        \multirow{2}{*}{Backbone}   & \multirow{2}{*}{Architecture} & Depth        & Params      & GFLOPs    & Throughput     \\
                                    &                               & \tiny{($t, \downarrow$)} & \tiny{(M, $\downarrow$)}  & (\tiny{$\downarrow$})  & \tiny{(img/s, $\uparrow$)} \\
        \midrule
                        ResNet50 & Vanilla   & 16.00 \textsubscript{$\pm$ 0.00} & 1.24  & 8.37  & 33.077 \\
                        ResNet50 & Recurrent & 16.00 \textsubscript{$\pm$ 0.00} & 0.11  & 8.37  & 32.898 \\
                        ResNet50 & BasicSSM  & 16.00 \textsubscript{$\pm$ 0.00} & 1.49  & 82.77 & 23.334 \\
        \rowcolor{Gray} ResNet50 & FROST     & 7.56  \textsubscript{$\pm$ 2.83} & 1.49  & 39.11 & 47.183 \\
        \midrule
                        ViT-B/16 & Vanilla   & 12.00 \textsubscript{$\pm$ 0.00} & 85.80 & 35.13 & 21.484 \\
                        ViT-B/16 & Recurrent & 12.00 \textsubscript{$\pm$ 0.00} & 7.83  & 35.13 & 21.976 \\
                        ViT-B/16 & BasicSSM  & 12.00 \textsubscript{$\pm$ 0.00} & 20.81 & 65.45 & 15.986 \\
        \rowcolor{Gray} ViT-B/16 & FROST     & 4.53  \textsubscript{$\pm$ 2.34} & 20.81 & 24.71 & 42.292 \\
        \bottomrule
    \end{tabular}
    \caption{
        Computational complexity comparison of recurrent baselines and FROST.
        ``Vanilla'' refers to the models that does not share parameters across depths.
        GFLOPs and throughput are measured under a fixed halting quantile ($q=0.5$) for adaptive methods.
        Throughput is measured on a single NVIDIA GeForce RTX 3090.
        \vspace{-2em}
    }
    \label{tab:complexity}
\end{table}

Applying the proposed ranking-based halting strategy to non-fractal baselines consistently resulted in unstable training and frequent collapse (Appendix~\ref{appndx:halting-collapse}). 
As a result, for Vanilla, Recurrent, and BasicSSM models, we report complexity metrics by assuming the average iteration number equals to the full depth.
Without the scale-consistent latent structure enforced by FROST, intermediate representations do not form valid approximations of the final output, causing the ranking signal to conflict with representation learning.

For ResNet50, the parameter count of FROST is relatively higher than that of the recurrent baseline, as the backbone itself is lightweight and the additional state-space components ($B$, $C$, and $D$) constitute a non-negligible fraction of the total FLOPs.
Nevertheless, ranking-based halting significantly reduces the average refinement depth of FROST from 16.0 to 7.56, leading to a substantial decrease in effective GFLOPs and a pronounced increase in inference throughput.
This result highlights that, even when parameter overhead is more visible, adaptive refinement enables FROST to demonstrate higher practical efficiency by reducing computation.

ViT-B/16 represents a more typical large-capacity setting, where the relative overhead of state-space components becomes marginal.
In this regime, FROST consistently improves all evaluated metrics compared to the Vanilla baseline, simultaneously reducing average depth, computational cost, and substantially increasing throughput.

These results highlight that the effectiveness of ranking-based adaptive computation critically depends on the underlying fractal representation geometry, and that FROST uniquely enables stable training and efficient inference under this halting strategy.

\subsection{Empirical Proxies}
\label{ssec:empirical-proxies}
While metrics such as cosine similarity and estimated Minkowski dimension may not constitute a complete characterization of fractal geometry, they provide practical and interpretable proxies for validating scale-consistent refinement in finite neural systems. 
Our empirical analysis focuses on these observables as necessary indicators of the theoretical properties derived in \S\ref{sec:frost}, rather than as exhaustive geometric descriptors.


\section{Conclusion}
\label{sec:conclusion}

In this work, we introduced FROST, a simple and principled framework that injects a fractal inductive bias into state-space models to enable variable-complexity inference via scale-consistent latent dynamics.
Empirical verification confirms that the proposed fractal inductive bias leads to the theoretically predicted convergence behaviors.

Beyond empirical gains, our results indicate that adaptive computation critically depends on the geometry of latent representations.
Constraining intermediate states to lie on a self-similar manifold provides a structural justification for early halting and highlights the importance of intrinsic representation consistency.

FROST currently relies on stationary transformations applied uniformly across tokens, which may limit expressivity for non-uniform or task-specific dynamics, motivating extensions to token-level adaptive refinement.
Applying the framework to larger-scale benchmarks also remains an important direction for evaluating the robustness and generality of the proposed geometric principles.



\bibliography{example_paper}
\bibliographystyle{icml2026}

\newpage
\appendix
\onecolumn

\section{Relationship with Efficient Methods}
\label{appndx:related-works}

\paragraph{Static vs. Adaptive Vision Backbones.}
Recent adaptations of state-space models to computer vision, such as Vision Mamba~\cite{vision-mamba} and VMamba~\cite{vmamba}, have demonstrated that SSMs can serve as effective general-purpose backbones.
However, these architectures typically adhere to a fixed columnar design, stacking distinct layers with independent parameters.
Consequently, they incur a constant inference cost for every input, treating simple and complex samples with an identical computational budget.

FROST diverges from this paradigm by employing a recursive, weight-tied formulation.
Instead of physically stacking layers, FROST iterates a stationary transformation to theoretically unbounded depths.
This structural difference allows FROST to dynamically adjust its computational expenditure, exiting early for easy instances while allocating more iterations for harder samples.
As a result, FROST breaks the rigid cost–performance trade-off of static backbones while retaining the modeling advantages of state-space models.

\paragraph{Adaptive Computation Time and Learnable Halting.}
Approaches such as Adaptive Computation Time (ACT)~\cite{act}, PonderNet~\cite{pondernet}, and LayerSkip~\cite{layerskip} formulate the halting decision as a policy learning problem. 
While effective for allocating computational resources based on input difficulty, these methods rely on an extrinsic halting policy trained via complex objectives (e.g., geometric priors or reinforcement learning). 
Crucially, they treat early stopping as a truncation of the computational graph. 
Consequently, intermediate states are often merely incomplete precursors to the final output, lacking a guarantee of semantic validity or distributional consistency with the fully processed representations.

In contrast, FROST relies on an intrinsic geometric constraint. 
Since the latent dynamics are governed by a contractive fractal mapping, all intermediate states reside on a shared, self-similar manifold. 
Halting in FROST corresponds to selecting a coarser resolution of a consistent representation rather than prematurely aborting an unfinished process. 
This allows for valid anytime predictions without the optimization overhead of auxiliary halting networks.

\paragraph{Mixture-of-Depths and Dynamic Routing.}
Dynamic routing frameworks, including Mixture-of-Depths (MoD)~\cite{mixture-of-depths}, achieve efficiency by selectively skipping blocks or tokens. 
In these architectures, varying the depth fundamentally alters the transformation path, often leading to qualitatively disjoint feature spaces (i.e., a layer-$N$ feature is distributionally distinct from a layer-$M$ feature). 
This necessitates specific routing logic to handle the discrepancy.

FROST reinterprets depth as iterative refinement rather than architectural routing. 
Because FROST employs a stationary transformation, successive iterations progressively sharpen the fidelity of the same underlying geometry. 
Thus, depth becomes a continuous control knob for representational precision, maintaining strict scale-consistency across all exit points.

\paragraph{Selective State Updates in SSMs.}
Recent efficient SSMs, such as Mamba-2~\cite{mamba2}, introduce selective scan mechanisms to skip or compress information along the sequence length dimension. 
These methods focus on token-level sparsity—deciding what information to propagate over time.

FROST addresses an orthogonal dimension of efficiency: computational density. 
It operates at the level of iterative recurrence, controlling how much processing is applied to the global state. 
Therefore, FROST's fractal refinement is complementary to token-selective mechanisms, and future work could potentially integrate both to achieve efficiency across both sequence length and computational depth.

\section{Theoretical Analysis}
\label{appndx:theory}

In this section, we provide rigorous proofs for the convergence, approximation error bounds, and gradient stability of the FROST framework.

\subsection{Proof of Convergence to a Unique Fixed Point}
\label{appndx:proof-convergence}

We establish that the iterative state update in FROST constitutes a contraction mapping, ensuring stable convergence.

\begin{definition}[FROST State Update]
Let $h_t \in \mathbb{R}^d$ be the hidden state at iteration $t$. We model the update rule as:
\begin{equation}
    h_{t+1} = \Phi(h_t, x) = \sigma(A_\lambda h_t + B x + b)
\end{equation}
where $\sigma$ is a 1-Lipschitz activation function (e.g., ReLU, Tanh).
\end{definition}

\noindent\textit{Note:} A function is said to be 1-Lipschitz if the distance between any two outputs is never larger than the distance between their corresponding inputs (i.e., $\|\sigma(u) - \sigma(v)\| \le \|u - v\|$), which intuitively means the function does not expand the signal energy.

\paragraph{Remark on Linearity and Generalization.}
In our practical implementation, the transition operator $A$ may involve non-linear gating mechanisms or input-dependent discretizations (as discussed in \S\ref{sec:frost}).
However, for the sake of theoretical tractability in this proof, we assume $A_\lambda$ represents the effective linearized transition matrix.
Note that this assumption is not restrictive; the contraction mapping principle holds for general non-linear operators as long as their global Lipschitz constant remains strictly less than 1.

\begin{theorem}[Contraction Mapping Principle]
If the spectral norm (or induced operator norm) of the transition matrix satisfies $\|A_\lambda\| < 1$, then the map $\Phi(\cdot, x)$ is a contraction with respect to $h$. 
By the Banach Fixed Point Theorem, the sequence $\{h_t\}$ converges to a unique fixed point $h^*$.
\end{theorem}

\begin{proof}
Consider two arbitrary states $h$ and $h'$. 
The distance between updated states is:
\begin{align}
    \| \Phi(h, x) - \Phi(h', x) \| &= \| \sigma(A_\lambda h + B x) - \sigma(A_\lambda h' + B x) \| \\
    &\le \| (A_\lambda h + B x) - (A_\lambda h' + B x) \| \quad (\text{by 1-Lipschitz property of } \sigma) \\
    &= \| A_\lambda (h - h') \| \\
    &\le \| A_\lambda \| \| h - h' \|
\end{align}
Let $L = \| A_\lambda \|$. 
In FROST, the parameterization guarantees $L \le \lambda < 1$. 
Thus, $\Phi$ is a contraction mapping, implying the existence of a unique limit $h^*$ such that $\lim_{t \to \infty} h_t = h^*$.
\end{proof}

\subsection{Approximation Error Bound and Early Exit}
\label{appndx:error-bound}

We derive the error bound for terminating the iteration at step $t$ (early exit or halt). 
This justifies our rationale that deeper iterations yield exponentially better approximations of the underlying manifold.

\begin{theorem}[Exponential Error Decay]
Let $h^*$ be the unique fixed point and $h_t$ be the state at iteration $t$. 
The approximation error satisfies:
\begin{equation}
    \| h_t - h^* \| \le \frac{L^t}{1-L} \| h_1 - h_0 \|
\end{equation}
where $L < 1$ is the contraction factor.
\end{theorem}

\begin{proof}
From the contraction property shown in \S\ref{appndx:proof-convergence}, we have $\| h_{t+1} - h_t \| \le L \| h_t - h_{t-1} \|$. 
By induction, $\| h_{t+1} - h_t \| \le L^t \| h_1 - h_0 \|$.
Using the triangle inequality for $m > t$:
\begin{align}
    \| h_m - h_t \| &\le \sum_{k=t}^{m-1} \| h_{k+1} - h_k \| \\
    &\le \sum_{k=t}^{m-1} L^k \| h_1 - h_0 \| \\
    &= L^t \left( \sum_{j=0}^{m-t-1} L^j \right) \| h_1 - h_0 \| \\
    &\le L^t \left( \sum_{j=0}^{\infty} L^j \right) \| h_1 - h_0 \| = \frac{L^t}{1-L} \| h_1 - h_0 \|
\end{align}
Taking the limit as $m \to \infty$, we get $h_m \to h^*$, yielding the bound:
\begin{equation}
    \| h^* - h_t \| \le \mathcal{O}(L^t)
\end{equation}
\end{proof}
\paragraph{Interpretation.} The error decays exponentially with depth $t$. 
This guarantees that even with early exiting strategies, the representation $h_t$ is structurally consistent with the optimal representation $h^*$, merely at a coarser resolution of precision.

\subsection{Gradient Dynamics and Stability Analysis}
\label{appndx:gradient-stability}

Finally, we analyze the backward pass dynamics. 
A common issue in recurrent networks is the exploding gradient problem. We show that FROST's contractive nature inherently prevents gradient explosion.

\begin{proposition}[Bounded Gradients]
The gradient of the state at step $t$ with respect to the initialization $h_0$ is bounded and decays over time, ensuring stability.
\end{proposition}

\begin{proof}
The Jacobian of the update rule at step $k$ is $J_k = \frac{\partial h_k}{\partial h_{k-1}} = \text{diag}(\sigma'(z_k)) W_\lambda$, where $z_k$ is the pre-activation vector.
The total Jacobian over $t$ steps is the product of individual Jacobians:
\begin{equation}
    \frac{\partial h_t}{\partial h_0} = \prod_{k=1}^{t} J_k = \prod_{k=1}^{t} \text{diag}(\sigma'(z_k)) W_\lambda
\end{equation}
Taking the spectral norm:
\begin{align}
    \left\| \frac{\partial h_t}{\partial h_0} \right\| &\le \prod_{k=1}^{t} \| \text{diag}(\sigma'(z_k)) \| \| W_\lambda \| \\
    &\le \prod_{k=1}^{t} 1 \cdot L \quad (\text{assuming } |\sigma'| \le 1 \text{ and } \|W_\lambda\| = L) \\
    &= L^t
\end{align}
Since $L < 1$, the influence of the initial state decays as $t \to \infty$. 
More importantly for training stability, the gradient norm is strictly bounded by $1$ (assuming normalized initialization), preventing the exploding gradient problem ($L > 1$) common in standard recurrent neural networks. 
This ensures that the learning signal remains stable across varying depths.
\end{proof}

\section{Architecture Details}
\label{appndx:architecture-details}

\paragraph{Shape-Consistency.} 
FROST expects the state transition function $A$ does not change the shape of output for stationary property. 
However, standard ResNet~\cite{resnet} includes pooling layers to widen the receptive fields.
To solve this mismatch, we slightly modified ResNet implementation; we we replaced max-pooling layer with dilated convolution with padding.
This still widen the receptive fields, and keep the shape consistent. 

\paragraph{$\lambda$ Range.}
$\lambda$ takes the role of scale factor, making the whole system contractive. 
Thus, the $\lambda$ parameter must be kepted positive. 
To this end, we reparameterized $\lambda$ with $\exp$-$\log$-$\lambda$.
This trick keeps $\lambda$ positive and the FROST system contractive while preventing the value explosion to the negatives.

\paragraph{Linearization of $B, C, D$.}
In \S\ref{sec:frost}, we assume the linearity of $B, C$, and $D$, so that the $B_\lambda(x) = \lambda^{1+H}$ and $D_\lambda(h) = D(h)$.
We implemented $B$ and $D$ as $1\times1$ convolution layers.
This implementation not only makes the practice close to the threory, but also saves computational complexity for efficiency. 

\paragraph{Gating Mechanism.}
Our expansion in Eq.~\ref{eq:discrete_realization} is can be interpreted as interpolation between $h_t$ and the residual term.
This can be efficiently implemented with linear mapping (e.g., $1\times1$ convolution) and sigmoid function.
This practical implementation also keep the system efficient and effective.

\section{Collapse with Ranking-based Halting}
\label{appndx:halting-collapse}

We analyze the instability observed when applying the proposed ranking-based halting strategy to non-fractal architectures, including standard recurrent model and BasicSSM.
Although ranking-based halting is designed to align computation allocation with task loss, we empirically found that these baselines consistently failed to converge across multiple seeds when trained with the same halting objective used for FROST.

The primary cause of this collapse lies in the mismatch between the halting signal and the structure of intermediate representations.
In non-fractal models, intermediate states produced at different iterations do not lie on a shared distribution and cannot be interpreted as progressively refined approximations of the final representation.
Consequently, enforcing a ranking constraint across iterations introduces conflicting optimization signals: the ranking loss encourages states to be ranked above later ones for certain samples, while the task loss simultaneously pushes later states to be strictly more informative.
This conflict leads to unstable gradients and rapid divergence during training.

In contrast, FROST enforces scale-consistent latent dynamics through its fractal inductive bias, ensuring that all intermediate states reside on a self-similar manifold.
Under this constraint, ranking-based halting no longer contradicts representation learning, as earlier and later states correspond to different resolutions of the same underlying representation.
As a result, the ranking objective reinforces, rather than interferes with, the primary task loss, enabling stable training and effective adaptive computation.

Due to this instability, we report complexity metrics for recurrent and BasicSSM baselines under fixed-depth execution, assuming the maximum number of iterations.
This analysis highlights that the effectiveness of ranking-based adaptive computation critically depends on the geometric consistency of latent representations, which is uniquely provided by FROST.

To rigorously explain the empirical collapse observed in baselines, we analyze the alignment between the gradients derived from the primary task loss ($\mathcal{L}_{\text{task}}$) and the ranking loss ($\mathcal{L}_{\text{rank}}$).

\begin{definition}[Gradient Conflict]
Let $\theta$ be the shared parameters of the model. 
We define the gradient conflict between the task objective and the ranking objective at iteration $t$ as the cosine similarity between their gradients:
\begin{equation}
    \mathcal{C}(\theta) = \frac{\nabla_\theta \mathcal{L}_{\text{task}} \cdot \nabla_\theta \mathcal{L}_{\text{rank}}}{\| \nabla_\theta \mathcal{L}_{\text{task}} \| \| \nabla_\theta \mathcal{L}_{\text{rank}} \|}
\end{equation}
If $\mathcal{C}(\theta) < 0$, the objectives are competing, leading to destructive interference and training instability.
\end{definition}

\paragraph{Case 1: Non-Fractal Architectures (Sequential Feature Extraction).}
In standard stacked architectures (e.g., Vanilla, standard SSMs), early layers are optimized to extract low-level features (e.g., edges, textures) that serve as precursors for later layers, not to represent the final semantic output.
Let $h_t$ be the state at layer $t$. 
The task loss backpropagated from the final layer $T$ encourages $h_t$ to maximize mutual information with the next layer's optimal input, denoted as vector $\mathbf{v}_\text{next}$.
Conversely, the ranking loss forces $h_t$ to directly classify the target $y$, pushing $h_t$ towards a semantic classification vector $\mathbf{v}_\text{class}$.
In general hierarchical models, $\mathbf{v}_\text{next} \not\approx \mathbf{v}_\text{class}$ (feature space $\neq$ logit space).
Therefore, $\nabla_{h_t} \mathcal{L}_{\text{task}}$ and $\nabla_{h_t} \mathcal{L}_{\text{rank}}$ point in significantly different directions, often yielding $\mathcal{C}(\theta) \ll 0$. 
This gradient conflict destabilizes the shared weights, causing the collapse observed in our experiments.

\paragraph{Case 2: FROST (Scale-Consistent Refinement).}
FROST is contractive, meaning all states $h_t$ lie in the basin of attraction of a unique fixed point $h^*$ which represents the optimal semantic representation.
\begin{equation}
    h_t = h^* + \epsilon_t, \quad \text{where } \|\epsilon_t\| \to 0 \text{ as } t \to \infty.
\end{equation}
Both the task loss (at $t \to \infty$) and the ranking loss (at step $t$) share the same global objective: minimizing the distance to the target manifold defined by $h^*$.
Improving the classification capability of $h_t$ (via $\mathcal{L}_{\text{rank}}$) inherently reduces $\|\epsilon_t\|$, which is congruent with the trajectory required by $\mathcal{L}_{\text{task}}$.
Consequently, $\nabla_\theta \mathcal{L}_{\text{task}}$ and $\nabla_\theta \mathcal{L}_{\text{rank}}$ are structurally aligned ($\mathcal{C}(\theta) > 0$), allowing the ranking objective to act as a valid auxiliary supervision rather than noise.

\section{The Use of Large Language Models}
\label{appndx:use-of-llms}

As per the ICML 2026 guidelines on the use of Large Language Models (LLMs), we disclose that an LLM was used for minor grammar corrections and polishing of the text to enhance readability and for searching related research to broaden the scope of the literature review. 
The LLM did not contribute to the research ideation, methodology, or core findings of the paper.


\end{document}